\documentclass[12pt]{article}
\usepackage[centertags]{amsmath}
\usepackage{amsfonts}
\usepackage{amssymb}
\usepackage{latexsym}
\usepackage{amsthm}
\usepackage{newlfont}
\usepackage{graphicx}
\usepackage{listings}
\usepackage{booktabs}
\usepackage{abstract}
\newlength{\defbaselineskip}
\newcommand{\setlinespacing}[1]%
           {\setlength{\baselineskip}{#1 \defbaselineskip}}

\newcommand{\actaqed}{\hfill $\actabox$}
{\medskip\noindent \textit{Proof of #1. }}%
{\actaqed \medskip}

\def\D{{\mathcal D}}

\def\M{{\mathcal M}}
\def \<{\langle}
\def\>{\rangle}

\def \e{\epsilon}
\def \de{\delta}
\def \ff{\varphi}

\def \sp{\operatorname{span}}

\def\bt{\beta}
\def\la{\lambda}
\def\a{\alpha}

\newtheorem{Theorem}{Theorem}[section]
\newtheorem{Lemma}{Lemma}[section]

\newtheorem{Proposition}{Proposition}[section]
\newtheorem{Remark}{Remark}[section]

\newtheorem{Corollary}{Corollary}[section]
\numberwithin{equation}{section}

\begin{document}
\title{{Convergence and rate of convergence of some greedy algorithms in convex optimization} }
\author{ V.N. Temlyakov \thanks{ University of South Carolina and Steklov Institute of Mathematics. Research was supported by NSF grant DMS-1160841 }} 
\maketitle
\begin{abstract}
{The paper gives a systematic study of the approximate versions of three greedy-type algorithms that are widely used in convex optimization. By approximate version we mean the one where some of evaluations are made with an error. Importance of such versions of greedy-type algorithms in convex optimization and in approximation theory was emphasized in previous literature.   }
\end{abstract}

\section{Introduction}

 We
  study approximate solutions of the optimization problem
\begin{equation}\label{1.1}
\inf_{x\in S} E(x)
\end{equation}
under certain assumptions on $E$ and $S$. In this paper we always assume that $E$ is a convex function defined on a Banach space $X$. We would like to construct an algorithm that after $m$ iterations provides a point $x_m$ such that $E(x_m)$ is close to the $\inf_{x\in S} E(x)$. There is an increasing interest in building such sparse approximate solutions using different greedy-type algorithms (see, for instance, \cite{Z}, \cite{SSZ}, \cite{CRPW}, \cite{Cl}, \cite{Ja1}, \cite{DHM}, \cite {JS}, \cite{TRD}, and \cite{Ja2}). This paper is a follow up to the papers \cite{Z}, \cite{T1}, \cite{T2}, and \cite{DT}. 

Three greedy-type algorithms and their approximate versions designed for convex optimization are discussed. We study here the Weak Relaxed Greedy Algorithm (WRGA(co)), the Relaxed $E$-Greedy Algorithm (REGA(co)), and the Weak Greedy Algorithm with Free Relaxation \newline (WGAFR(co)). The used above names of these algorithms are from approximation theory. The WRGA(co) is the approximation theory analog of the classical Frank-Wolfe Algorithm, 
introduced in \cite{FW} and studied in many papers (see, for instance, \cite{DR}, \cite{DH}, \cite{Cl}, \cite{Ja1}, \cite{Ja2}, \cite{DHM}). This algorithm was rediscovered in 
statistics and approximation theory in \cite{B} and \cite{Jo} (see \cite{Tbook} for further discussion). The REGA(co) was introduced in \cite{Z} under the name Sequential greedy approximation. The WGAFR(co) was studied in \cite{T1}. 

The novelty of this paper is in a systematic study of the approximate versions of the above three greedy-type algorithms. By approximate version we mean the one where some of evaluations are made with an error. Importance of such versions of greedy-type algorithms in convex optimization and in approximation theory was emphasized respectively in \cite{Z}, \cite{DT} and in \cite{T7}. We now proceed to a detailed discussion of our results. 

Let $X$ be a Banach space.
A system $\D:=\{g\}$, $g\in X$, is called a {\it symmetric dictionary} if $\|g\|:=\|g\|_X=1$, $g\in \D$ implies $-g\in \D$ for all $g\in \D$
and the closure of $\sp(\D)$ coincides with $X$. 
We impose some conditions on the minimizer $x^*$ of the problem (\ref{1.1}). In some theorems 
we assume that the point $x^*$ of the minimum of $E(x)$ belongs to the closure of the convex hull of a given symmetric dictionary $\D$, which is denoted by $A_1(\D)$. In other words this assumption means $\inf_{x\in S}E(x) = \inf_{x\in A_1(\D)}E(x)$. Also, we might be interested in finding $\inf_{x\in A_1(\D)}E(x)$ without the above assumption. In both cases we can set $S=A_1(\D)$. We use this assumption to design different greedy-type algorithms which build at the $m$th iteration an approximant of $x^*$ as an $m$-term polynomial with respect to $\D$. Some of these algorithms alike gradient methods use the derivative $E'(x)$. Other algorithms use only function $E$ values. We discuss these algorithms in Section 2 where we present some known convergence and rate of convergence results. The algorithms in Section 2 are assumed to perform exact evaluations of $E(x)$ (or $E'(x)$). Here is a typical algorithm of this family. The following algorithm was introduced in \cite{Z} under the name Sequential greedy approximation. As in \cite{T1} and \cite{T2} we use the names of algorithms from approximation theory. 

{\bf Relaxed $E$-Greedy Algorithm (REGA(co)).} 
We define   $G_0 := 0$. Then, for each $m\ge 1$ we have the following inductive definition.

(1) $\varphi_m   \in \D$ is any element and $0\le \lambda_m \le 1$ is a number satisfying  (assuming existence)
$$
E((1-\la_m)G_{m-1} + \la_m\varphi_m) = \inf_{0\le \la\le 1;g\in\D}E((1-\la)G_{m-1} + \la g)
$$
and define
$$
G_m:=  (1-\la_m)G_{m-1} + \la_m\varphi_m.
$$

In our algorithms we always begin with 
$G_0=0$ and build approximants $G_1,G_2,\dots$. A natural for us domain 
for minimization problem (\ref{1.1}) is
$$
D:=\{x:E(x)\le E(0)\}.
$$
We assume that $D$ is bounded.
The following convergence and rate of convergence results for the REGA(co) were established in \cite{DT}.

\begin{Theorem}\label{T1.1} Let $E$ be a uniformly smooth on $D\cap A_1(\D)$ convex function.   Then, for the REGA(co) we have 
$$
\lim_{m\to \infty} E(G_m) =\inf_{x\in A_1(\D)}E(x).
$$
\end{Theorem}

\begin{Theorem}\label{T1.2} Let $E$ be a uniformly smooth on $D\cap A_1(\D)$ convex function with modulus of smoothness $\rho(E,u) \le \gamma u^q$, $1<q\le 2$. Then, for the REGA(co) we have  
$$
E(G_m)-\inf_{x\in A_1(\D)}E(x) \le  C(q,\gamma)m^{1-q} ,
$$
with a positive constant $C(q,\gamma)$ which may depend only on $q$ and $\gamma$.
\end{Theorem}
In the case $\rho(E,u)\le \gamma u^2$ Theorems \ref{T1.1} and \ref{T1.2} were proved in \cite{Z}.

It is clear from the definition of the REGA(co) that all approximants $G_m$ belong to $A_1(\D)$. Thus, the REGA(co) can only be used for solving problem (\ref{1.1}) with $S=A_1(\D)$.  
 
 The most important results of the paper are in Section 3. Very often we cannot calculate values of $E$ exactly. Even if we can evaluate $E$ exactly we may not be able to find the exact value of, say, the $\inf_{0\le \la\le1;g\in\D} E((1-\la)G_{m-1}+\la g)$ in the REGA(co). This motivates us to study the corresponding modifications of the algorithms discussed in Section 2. In Section 3 we assume that instead of exact evaluation of $\min_{0\le \la \le 1} E((1-\la)G_{m-1}+\la g)$ we can only   do an approximate  evaluation with error $\de_{m-1}$. Here is the corresponding modification of the REGA(co). The following algorithm which is an approximate variant of the REGA(co) was introduced in \cite{Z}. 

 {\bf Relaxed $E$-Greedy Algorithm with errors $\{\de_k\}$ (REGA($\{\de_k\}$)).} Let $\de_k\in(0,1]$, $k=0,1,2,\dots$.
We define   $G_0 := 0$. Then, for each $m\ge 1$ we have the following inductive definition.

(1) $\varphi_m   \in \D$ is any element and $0\le \lambda_m \le 1$ is a number satisfying   
$$
E((1-\la_m)G_{m-1} + \la_m\varphi_m) \le \inf_{0\le \la\le 1;g\in\D}E((1-\la)G_{m-1} + \la g) +\de_{m-1}
$$
and define
$$
G_m:=  (1-\la_m)G_{m-1} + \la_m\varphi_m.
$$

Convergence of the REGA($\{\de_k\}$) under conditions $\rho(E,u)\le \gamma u^2$ and $\de_k\to 0$ as $k\to \infty$ was established in \cite{Z}. We prove in Section 3 convergence of the REGA($\{\de_k\}$) under conditions $\rho(E,u)=o(u)$ and $\de_k\to 0$ as $k\to \infty$.
Also we prove in Section 3 the following rate of convergence result for the REGA($\{\de_k\}$). 
\begin{Theorem}\label{T1.5} Let $E$ be a uniformly smooth on $A_1(\D)$ convex function with modulus of smoothness $\rho(E,u) \le \gamma u^q$, $1<q\le 2$. Then, for the REGA($\{\de_k\}$) with $\de_k \le ck^{-q}$ we have  
$$
E(G_m)-b \le  C(q,\gamma,E,c) m^{1-q},  
$$
where $b:=\inf_{f\in A_1(\D)}E(x)$.
\end{Theorem}
In the case $q=2$ Theorem \ref{T1.5} was proved in \cite{Z}. 
In this paper we extend results on the REGA($\{\de_k\}$) from \cite{Z} in different directions. First, we prove the convergence result under assumption $\rho(E,u)=o(u)$ which is much weaker than $\rho(E,u)\le \gamma u^2$. Second, we prove the rate of convergence result under conditions $\rho(E,u) \le \gamma u^q$, $1<q\le 2$ and $\de_k \le ck^{-q}$. Third, along with the REGA($\{\de_k\}$) we study other algorithms: Weak Relaxed Greedy Algorithm with errors $\{\de_k\}$ (WRGA($\{\de_k\}$)), Weak Greedy Algorithm with Free Relaxation and errors $\{\de_k\}$ (WGAFR($\{\de_k\}$)). The variant of the above algorithms when $\de_k=\de$ for all $k$ was studied in \cite{DT}. This case is realistic from the practical point of view.   For $\de>0$ we obtained in \cite{DT} the same asymptotic bound for the error as in the case $\de=0$ but for a limited ($\le \de^{-1/q} $) number of iterations. 

One of important contributions of this paper is that it gives a dimension independent analysis of unconstrained convex optimization. For that purpose we use an algorithm with free relaxation -- the WGAFR($\{\delta_k\}$), which we describe momentarily. An important difference between these algorithms and the one introduced and studied in \cite{Z} -- REGA($\{\delta_k\}$) -- is that the REGA($\{\delta_k\}$) is limited to convex combinations $(1-\lambda_m)G_{m-1} +\lambda_m\ff_m$ and, therefore, it is only applicable for minimization over $A_1(\D)$. Also, we point out that our analysis is different from that in \cite{Z}. In both approaches the reduction $E(G_{m-1})-E(G_m)$ at one iteration is analyzed. 
We analyze it using $\ff_m$ satisfying the greedy condition:
$$
\<-E'(G_{m-1}),\varphi_m\> \ge t_m  \sup_{g\in\D}\<-E'(G_{m-1}),g\>.
$$
In \cite{Z} the averaging technique is used.  

{\bf Weak Greedy Algorithm with Free Relaxation and errors $\{\de_k\}$ \newline (WGAFR($\{\de_k\}$)).} 
Let $\tau:=\{t_m\}_{m=1}^\infty$, $t_m\in[0,1]$, be a weakness  sequence and let $\de_k\in[0,1]$, $k=0,1,2,\dots$. We define   $G_0  := 0$. Then for each $m\ge 1$ we have the following inductive definition.

(1) $\varphi_m   \in \D$ is any element satisfying
\begin{equation}\label{3.12}
\<-E'(G_{m-1}),\varphi_m\> \ge t_m  \sup_{g\in\D}\<-E'(G_{m-1}),g\>.
\end{equation}

(2) Find $w_m$ and $ \lambda_m$ such that
$$
E((1-w_m)G_{m-1} + \la_m\varphi_m) \le \inf_{ \la,w}E((1-w)G_{m-1} + \la\varphi_m) +\de_{m-1}
$$
and define
$$
G_m:=   (1-w_m)G_{m-1} + \la_m\varphi_m.
$$
In the case $\de_k=\de >0$, $k=0,1,\dots$, the WGAFR($\de$) was introduced and studied in \cite{DT}.

\section {Greedy algorithms in convex optimization} 

In this section we formulate some of the known results on the WRGA(co) and the WGAFR(co), which are used later in Section 3.   In \cite{T1}, \cite{T2} the problem of sparse approximate solutions to convex optimization problems was considered.  
  We begin with some notations and definitions. 
We 
assume that the set
$$
D:=\{x:E(x)\le E(0)\}
$$
is bounded.
For a bounded set $S$ define the modulus of smoothness of $E$ on $S$ as follows
\begin{equation}\label{2.1}
\rho(E,u):=\rho(E,S,u):=\frac{1}{2}\sup_{x\in S, \|y\|=1}|E(x+uy)+E(x-uy)-2E(x)|.
\end{equation}
We say that $E$ is uniformly smooth on $S$ if $\rho(E,S,u)/u\to 0$ as $u\to 0$. 

We assume that $E$ is Fr{\'e}chet differentiable. Then convexity of $E$ implies that for any $x,y$ 
\begin{equation}\label{2.2}
E(y)\ge E(x)+\<E'(x),y-x\>
\end{equation}
or, in other words,
\begin{equation}\label{2.3}
E(x)-E(y) \le \<E'(x),x-y\> = \<-E'(x),y-x\>.
\end{equation} 
The following simple lemma holds.
\begin{Lemma}\label{L2.1} Let $E$ be Fr{\'e}chet differentiable convex function. Then the following inequality holds for $x\in S$
\begin{equation}\label{2.4}
0\le E(x+uy)-E(x)-u\<E'(x),y\>\le 2\rho(E,u\|y\|).  
\end{equation}
\end{Lemma}

The following two greedy algorithms were studied in \cite{T1}.

{\bf Weak Greedy Algorithm with Free Relaxation  (WGAFR(co)).} 
Let $\tau:=\{t_m\}_{m=1}^\infty$, $t_m\in[0,1]$, be a weakness  sequence. We define   $G_0  := 0$. Then for each $m\ge 1$ we have the following inductive definition.

(1) $\varphi_m   \in \D$ is any element satisfying
$$
\<-E'(G_{m-1}),\varphi_m\> \ge t_m  \sup_{g\in\D}\<-E'(G_{m-1}),g\>.
$$

(2) Find $w_m$ and $ \lambda_m$ such that
$$
E((1-w_m)G_{m-1} + \la_m\varphi_m) = \inf_{ \la,w}E((1-w)G_{m-1} + \la\varphi_m)
$$
and define
$$
G_m:=   (1-w_m)G_{m-1} + \la_m\varphi_m.
$$

 {\bf Weak Relaxed Greedy Algorithm (WRGA(co)).} 
We define   $G_0:=G^{r,\tau}_0 := 0$. Then, for each $m\ge 1$ we have the following inductive definition.

(1) $\varphi_m := \varphi^{r,\tau}_m \in \D$ is any element satisfying
$$
\<-E'(G_{m-1}),\varphi_m - G_{m-1}\> \ge t_m \sup_{g\in \D} \<-E'(G_{m-1}),g - G_{m-1}\>.
$$

(2) Find $0\le \lambda_m \le 1$ such that
$$
E((1-\la_m)G_{m-1} + \la_m\varphi_m) = \inf_{0\le \la\le 1}E((1-\la)G_{m-1} + \la\varphi_m)
$$
and define
$$
G_m:= G^{r,\tau}_m := (1-\la_m)G_{m-1} + \la_m\varphi_m.
$$

Convergence and rate of convergence results for the above two algorithms were proved in \cite{T1}.
For instance, the following theorem from \cite{T1} gives the rate of convergence of the WGAFR(co).  
\begin{Theorem}\label{T2.1} Let $E$ be a uniformly smooth on $D$ convex function with modulus of smoothness $\rho(E,u)\le \gamma u^q$, $1<q\le 2$. Take a number $\e\ge 0$ and an element  $f^\e$ from $D$ such that
$$
E(f^\e) \le \inf_{x\in D}E(x)+ \e,\quad
f^\e/A(\e) \in A_1(\D),
$$
with some number $A(\e)\ge 1$.
Then we have for the WGAFR(co) ($p:=q/(q-1)$)
$$
E(G_m)-\inf_{x\in D}E(x) \le  \e+ C_1(E,q,\gamma)A(\e)^q\left(C_2(E,q,\gamma)+\sum_{k=1}^mt_k^p\right)^{1-q} . 
$$
\end{Theorem}
It is clear from 
the definitions of WGAFR(co) and WRGA(co) that these algorithms use the derivative of $E$.  

It follows from the definitions of all algorithms described in this section that they have the monotonicity property 
$$
E(G_0)\ge E(G_1) \ge E(G_2) \ge \cdots  
$$
In particular, this implies that $G_m\in D$ for all $m$. 

We note that two of the above algorithms, namely, the WRGA(co) and the REGA(co) provide approximants $G_m$ which belong to $A_1(\D)$ for all $m$. Therefore, for these algorithms $G_m\in D\cap A_1(\D)$. Thus, in theorems for these two algorithms we may make our smoothness assumptions on the domain $D\cap A_1(\D)$. 
 
  Results for   REGA(co) can be obtained from the corresponding lemmas used in the study of the   WRGA(co) (see \cite{DT}).
  The following lemma was proved in \cite{T1}.
\begin{Lemma}\label{L2.2} Let $E$ be a uniformly smooth convex function with modulus of smoothness $\rho(E,u)$. Then, for any $f\in A_1(\D)$ we have for iterations of the WRGA(co)
$$
E(G_m) \le E(G_{m-1} )+ \inf_{0\le\la\le 1}(-\la t_m (E(G_{m-1} )-E(f))+ 2\rho(E, 2\la)), 
$$
for $m=1,2,\dots $.
\end{Lemma}

In the proof of this lemma we did not use a specific form of $G_{m-1}$ as the one generated by the WRGA(co),  we only used that $G_{m-1}\in D$. It was pointed out in \cite{DT} that the above lemma can be reformulated in the form.

\begin{Lemma}\label{L2.3} Let $E$ be a uniformly smooth convex function with modulus of smoothness $\rho(E,u)$. For $G\in D$ we apply the $m$th iteration of the WRGA(co): 

(1) $\varphi_m  \in \D$ is any element satisfying
$$
\<-E'(G),\varphi_m - G\> \ge t_m \sup_{g\in \D} \<-E'(G),g - G\>.
$$

(2) Find $0\le \lambda_m \le 1$ such that
$$
E((1-\la_m)G + \la_m\varphi_m) = \inf_{0\le \la\le 1}E((1-\la)G + \la\varphi_m)
$$
and define
$$
G_m  := (1-\la_m)G + \la_m\varphi_m.
$$

 Then, for any $f\in A_1(\D)$ we have  
$$
E(G_m) \le E(G )+ \inf_{0\le\la\le 1}(-\la t_m (E(G )-E(f))+ 2\rho(E, 2\la)),
\quad m=1,2,\dots .
$$
\end{Lemma}

\section{Approximate greedy algorithms for convex optimization}

We begin with a discussion of the WRGA(co) and the REGA(co) and introduce their approximate versions. The first step of the $m$th iteration of the WRGA(co) uses a weakness parameter $t_m$ which makes it feasible in case $t_m<1$ and allows some relative error in estimating $\sup_{g\in \D} \<-E'(G_{m-1}),g - G_{m-1}\>$. 
We concentrate on a modification of the second step of the WRGA(co). Very often we cannot calculate values of $E$ exactly. Even in case we can evaluate $E$ exactly we may not be able to find the exact value of the $\inf_{0\le \la\le 1}E((1-\la)G_{m-1} + \la\varphi_m)$. This motivates us to study the following modification of the WRGA(co).
In the case $\de_k=\de$, $k=0,1,\dots$ this algorithm was studied in \cite{DT}. 

{\bf Weak Relaxed Greedy Algorithm with errors $\{\de_k\}$ (WRGA($\{\de_k\}$)).} Let $\de_k\in (0,1]$, $k=0,1,2,\dots$.
We define   $G_0:=G^{\{\de_k\},\tau}_0 := 0$. Then, for each $m\ge 1$ we have the following inductive definition.

(1) $\varphi_m := \varphi^{\{\de_k\},\tau}_m \in \D$ is any element satisfying
$$
\<-E'(G_{m-1}),\varphi_m - G_{m-1}\> \ge t_m \sup_{g\in \D} \<-E'(G_{m-1}),g - G_{m-1}\>.
$$

(2) Find $0\le \lambda_m \le 1$ such that
$$
E((1-\la_m)G_{m-1} + \la_m\varphi_m) \le \inf_{0\le \la\le 1}E((1-\la)G_{m-1} + \la\varphi_m)+\de_{m-1}
$$
and define
$$
G_m:=G_m^{\{\de_k\},\tau}:=   (1-\la_m)G_{m-1} + \la_m\varphi_m.
$$

In the same way we modify the REGA(co). This algorithm was studied in \cite{Z}. 

 {\bf Relaxed $E$-Greedy Algorithm with errors $\{\de_k\}$ (REGA($\{\de_k\}$)).} Let $\de_k\in(0,1]$, $k=0,1,2,\dots $.
We define   $G_0 := 0$. Then, for each $m\ge 1$ we have the following inductive definition.

(1) $\varphi_m   \in \D$ is any element and $0\le \lambda_m \le 1$ is a number satisfying   
$$
E((1-\la_m)G_{m-1} + \la_m\varphi_m) \le \inf_{0\le \la\le 1;g\in\D}E((1-\la)G_{m-1} + \la g) +\de_{m-1}.
$$
Define
$$
G_m:=  (1-\la_m)G_{m-1} + \la_m\varphi_m.
$$

We begin with two theorems on convergence and rate of convergence of the WRGA($\{\de_k\}$). 

\begin{Theorem}\label{T3.0} Let $E$ be a uniformly smooth on $A_1(\D)$ convex function. Suppose that a sequence  $\{\de_k\}$ is such that $\de_k\to0$ as $k\to\infty$.
Then for the WRGA($\{\de_k\}$) with $t_k=t$, $t\in (0,1]$, $k=1,2,\dots$, we have 
$$
\lim_{m\to\infty}E(G_m) = \inf_{f\in A_1(\D)}E(x).    
$$
\end{Theorem}

\begin{Theorem}\label{T3.1} Let $E$ be a uniformly smooth on $A_1(\D)$ convex function with modulus of smoothness $\rho(E,u) \le \gamma u^q$, $1<q\le 2$. Then, for a sequence $\tau := \{t_k\}_{k=1}^\infty$, $t_k =t$, $k=1,2,\dots,$ and a sequence $\{\de_k\}$, $\de_k\le c(k+1)^{-q}$, $k=0,1,2\dots$ we have for the WRGA($\{\de_k\}$)
$$
E(G_m)-b \le  C(q,\gamma,t,E,c) m^{1-q},  
$$
where $b:=\inf_{f\in A_1(\D)}E(x)$.
\end{Theorem}
\begin{proof} In the proofs of both theorems we will use the following analog of Lemma 3.1 from \cite{T1}. 
\begin{Lemma}\label{L3.1} Let $E$ be a uniformly smooth on $A_1(\D)$ convex function with modulus of smoothness $\rho(E,u)$. Then, for any $f\in A_1(\D)$ we have for the WRGA($ \{\de_k\}$) for $m=1,2,\dots $
$$
E(G_m) \le E(G_{m-1} )+ \inf_{0\le\la\le 1}(-\la t_m (E(G_{m-1} )-E(f))+ 2\rho(E, 2\la))+\de_{m-1} 
$$
and therefore
\begin{equation}\label{3.1}
E(G_m) \le E(G_{m-1} )+ \inf_{0\le\la\le 1}(-\la t_m (E(G_{m-1} )-b)+ 2\rho(E, 2\la))+\de_{m-1}, 
\end{equation}
where $b:=\inf_{f\in A_1(\D)}E(x)$.
\end{Lemma}
\begin{proof} We have
$$
G_m := (1-\la_m)G_{m-1}+\la_m\varphi_m = G_{m-1}+\la_m(\varphi_m-G_{m-1})
$$
and
$$
E(G_m) \le \inf_{0\le \la\le 1}E(G_{m-1}+\la(\varphi_m-G_{m-1}))+\de_{m-1}.
$$
By Lemma \ref{L2.1} we have for any $\la$
$$
E(G_{m-1}+\la (\varphi_m-G_{m-1}))
$$
\begin{equation}\label{3.2}
  \le E(G_{m-1}) - \la\<-E'(G_{m-1}),\varphi_m-G_{m-1}\> + 2 \rho(E,2\la)
\end{equation}
and by (1) from the definition of the WRGA($\{\de_k\}$) and Lemma 2.2 from \cite{T1} (see also Lemma 6.10, p. 343 of \cite{Tbook}) we get
$$
\<-E'(G_{m-1}),\varphi_m-G_{m-1}\> \ge t_m \sup_{g\in \D} \<-E'(G_{m-1}),g-G_{m-1}\> =
$$
$$
t_m\sup_{\phi\in A_1(\D)} \<-E'(G_{m-1}),\phi-G_{m-1}\> \ge t_m   \<-E'(G_{m-1}),f-G_{m-1}\>.
$$
By (\ref{2.3})   we obtain
$$
\<-E'(G_{m-1}),f-G_{m-1}\> \ge E(G_{m-1})-E(f).
$$
Thus,  
$$
E(G_m) \le \inf_{0\le\la\le1} E(G_{m-1}+ \la(\varphi_m-G_{m-1})) +\de_{m-1} 
$$
\begin{equation}\label{3.3}
\le E(G_{m-1}) + \inf_{0\le\la\le1}(-\la t_m (E(G_{m-1})-E(f)) + 2\rho(E,2\la))+\de_{m-1},  
\end{equation}
which proves the lemma.
\end{proof}

We now proceed to the proof of Theorem \ref{T3.0}. It will follow from the above Lemma \ref{L3.1} and Lemma \ref{L3.0} below.

\begin{Lemma}\label{L3.0} Let $\rho(u)$ be a nonnegative convex on $[0,1]$ function with the property 
$\rho(u)/u\to0$ as $u\to 0$. Assume that a nonnegative sequence $\{\de_k\}$ is such that $\de_k\to0$ as $k\to\infty$. Suppose that a nonnegative sequence $\{a_k\}_{k=0}^\infty$ satisfies  the inequalities
$$
a_m\le a_{m-1} +\inf_{0\le\la\le1}(-\la va_{m-1}+B\rho(\la)) +\de_{m-1},\quad m=1,2,\dots,
$$
with positive numbers $v$ and $B$. Then
$$
\lim_{m\to\infty} a_m =0.
$$
\end{Lemma}
\begin{proof} We carry out the proof under assumption that $\rho(u)>0$ for $u>0$. Otherwise, if $\rho(u)=0$ for $u\in (0,u_0]$ then
$$
a_m\le (a_0+\de_0)(1-u_0v)^{m-1}+\de_1(1-u_0v)^{m-2}+\cdots+\de_{m-1} \to 0\quad \text{as}\quad m\to\infty.
$$
Denote
$$
\bt_{m-1}:=-\inf_{0\le\la\le1}(-\la va_{m-1}+B\rho(\la)).
$$
It is clear that $\bt_{m-1}\ge 0$. We divide the set of natural numbers into two sets:
$$
\M_1:=\{m:\bt_{m-1}\le 2\de_{m-1}\};\qquad \M_2:=\{m:\bt_{m-1}>2\de_{m-1}\}.
$$
The set $\M_1$ can be either finite or infinite. First, consider the case of infinite $\M_1$. 
Let
$$
\M_1=\{m_k\}_{k=1}^\infty,\quad m_1<m_2<\dots .
$$
For any $m\in\M_2$ we have
$$
a_m\le a_{m-1} -\bt_{m-1}+\de_{m-1} < a_{m-1}-\de_{m-1} \le a_{m-1}. 
$$
Thus, the sequence $\{a_m\}$ is monotone decreasing on $(m_{k-1},m_k)$. Also, we have
\begin{equation}\label{d3.1}
a_{m_{k-1}} \le a_{m_{k-1}-1}+\de_{m_{k-1}-1}.
\end{equation}
It is clear from (\ref{d3.1}), monotonicity of $\{a_m\}$ on $(m_{k-1},m_k)$ and the property 
$\de_k\to0$ as $k\to\infty$ that it is sufficient to prove that 
$$
\lim_{k\to\infty} a_{m_{k}-1}=0.
$$
For $m\in\M_1$ we have $\bt_{m-1}\le 2\de_{m-1}$. Let $\la_1(m)$ be a nonzero solution to the equation 
$$
\la v a_{m-1} = 2B\rho(\la) \quad \text{or}\quad \frac{\rho(\la)}{\la} = \frac{va_{m-1}}{2B}.
$$
If $\la_1(m)\le 1$ then
$$
-\bt_{m-1} \le -\la_1(m)va_{m-1}+B\rho(\la_1(m))=-B\rho(\la_1(m)).
$$
Thus, 
$$
B\rho(\la_1(m)) \le \bt_{m-1} \le 2\de_{m-1}.
$$
Therefore, using that $\rho(u)>0$ for $u>0$, we obtain $\la_1(m)\to 0$ as $m\to\infty$. Next,
$$
a_{m-1} = \frac{2B}{v}\frac{\rho(\la_1(m))}{\la_1(m)} \to 0 \quad \text{as}\quad m\to\infty.
$$
If $\la_1(m)>1$ then by monotonicity of $\rho(u)/u$ for all $\la\le \la_1(m)$ we have $\la va_{m-1}\ge 2B\rho(\la)$. Specifying $\la=1$ we get
$$
-\bt_{m-1} \le -\frac{1}{2}va_{m-1}.
$$
Therefore, $a_{m-1}\le (2/v)\bt_{m-1} \le (4/v)\de_{m-1} \to 0$ as $m\to\infty$.

Second, consider the case of finite $\M_1$. Then there exists $m_0$ such that for all $m\ge m_0$ we have
\begin{equation}\label{d3.2}
a_m\le a_{m-1} -\bt_{m-1} +\de_{m-1} \le a_{m-1}-\frac{1}{2} \bt_{m-1}.
\end{equation} 
The sequence $\{a_m\}_{m\ge m_0}$ is monotone decreasing and therefore it has a limit $\alpha\ge 0$. We prove that $\alpha=0$ by contradiction. Suppose $\alpha>0$. Then $a_{m-1}\ge \alpha$ for $m\ge m_0$. It is clear that for $m\ge m_0$ we have $\bt_{m-1} \ge c_0>0$. This together with (\ref{d3.2}) contradict to our assumption that $a_m\ge\alpha$, $m\ge m_0$. 

We now complete the proof of Theorem \ref{T3.0}. Denote
$$
a_k:=E(G_k)-b \ge 0,\quad b:=\inf_{x\in A_1(\D)}E(x).
$$
Set $v:=t$, $\rho(u):=\rho(E,2u)$, $B=2$. Then by Lemma \ref{L3.1} the nonnegative sequence $\{a_k\}$ satisfies the inequalities from Lemma \ref{L3.0}. It remains to apply Lemma \ref{L3.0}. 
\end{proof}

We proceed to the proof of Theorem \ref{T3.1}. Denote as above
$$
a_k:=E(G_k)-b \ge 0.
$$
Then taking into account that $\rho(E,u)\le \gamma u^q$ we get from Lemma \ref{L3.1}
\begin{equation}\label{3.4}
a_m\le a_{m-1} +\inf_{0\le \la\le 1}(-\la ta_{m-1}+2\gamma(2\la)^q)+\de_{m-1}.
\end{equation}

We now prove a lemma that gives the rate of decay of a sequence satisfying (\ref{3.4}).
\begin{Lemma}\label{L3.2} Suppose a nonnegative sequence $a_0,a_1,\dots$ satisfies the inequalities for $m=1,2,\dots$
\begin{equation}\label{3.5}
a_m\le a_{m-1} +\inf_{0\le \la\le 1}(-\la va_{m-1}+B\la^q)+\de_{m-1},  \quad \de_{m-1}\le cm^{-q},
\end{equation}
where   $q\in (1,2]$, $v\in(0,1]$, and $B>0$. Then
$$
a_m\le C(q,v,B,a_0,c) m^{1-q},\qquad C(q,v,B,a_0,c) \le C'(q,B,a_0,c)v^{-q} . 
$$
\end{Lemma}
\begin{proof}
In particular, (\ref{3.5}) implies that 
\begin{equation}\label{3.6}
a_m\le a_{m-1}+\de_{m-1}.
\end{equation}
  Then for all $m$ we have
$$
a_m\le a_0+C_1(q,c),\quad C_1(q,c):= c\sum_{k=0}^\infty (k+1)^{-q}.
$$
Denote $\la_1$ a nonzero solution of the equation
\begin{equation}\label{3.7}
 \la v a_{m-1}=  2B\la^q, \quad  \la_1= \left(\frac{va_{m-1}}{2B}\right)^{\frac{1}{q-1}}.
\end{equation}
If $\la_1\le 1$ then 
$$
\inf_{0\le \la\le 1}(-\la va_{m-1}+ B\la^q)\le -\la_1 va_{m-1}+ B\la_1^q 
$$
$$
=-\frac{1}{2}\la_1va_{m-1} = -C_1'(q,B)v^pa_{m-1}^p,\quad p:=\frac{q}{q-1}.
$$
If $\la_1> 1$ then for all $\la\le\la_1$ we have $\la v a_{m-1}\ge  2B\la^q$ and specifying $\la=1$ we get
$$
\inf_{0\le \la\le 1}(-\la va_{m-1}+B\la^q)\le-\frac{1}{2}va_{m-1} 
$$
$$
\le -\frac{1}{2}va_{m-1}^p(a_0+C_1(q,c))^{1-p} = -C_1(q,a_0,c)va_{m-1}^p.
$$
Setting $C_2:=C_2(q,v,B,a_0,c):=\min(C_1'(q,B)v^p,C_1(q,a_0,c)v)$ we obtain
from (\ref{3.5})
\begin{equation}\label{3.8}
a_m\le a_{m-1}- C_2a_{m-1}^p +\de_{m-1},\quad C_2\ge C_2'(q,B,a_0,c)v^p.
\end{equation}

We now need one more technical lemma. 
\begin{Lemma}\label{L3.2d} Let $q\in (1,2]$, $p:=\frac{q}{q-1}$. Assume that a sequence 
$\{\de_k\}_{k=0}^\infty$ is such that $\de_k\ge 0$ and $\de_k\le c(k+1)^{-q}$. Suppose a nonnegative sequence $\{a_k\}_{k=0}^\infty$ satisfies the inequalities
\begin{equation}\label{d3.3}
a_m\le a_{m-1}-wa_{m-1}^p +\de_{m-1},\qquad m=1,2,\dots,
\end{equation}
with a positive number $w\in(0,1]$. Then
$$
a_m\le C(q,c,w,a_0)m^{1-q},\quad m=1,2,\dots,\quad C(q,c,w,a_0)\le C'(q,c,a_0)w^{-\frac{1}{p-1}}.
$$
\end{Lemma}
\begin{proof} Lemma \ref{L3.2d} is a simple corollary of the following known lemma.
Lemma \ref{L3.4} below is a more general version of Lemma 2.1 from \cite{T1a} (see also Remark 5.1 in \cite{T7} and Lemma 2.37 on p. 106 of \cite{Tbook}).
\begin{Lemma}\label{L3.4}  Let three positive numbers $\a < \beta $, $A$   be given and let a sequence $\{a_n\}_{n=0}^\infty$ have the following properties:  $ a_0<A$ and  we have for all $n\ge 1$
 \begin{equation}\label{3.20}
 a_n\le a_{n-1}+An^{-\a};  
\end{equation}
 if for some $\nu $ we have
$$
a_\nu \ge A\nu^{-\a}
$$
then
\begin{equation}\label{3.21}
a_{\nu + 1} \le a_\nu (1- \beta/\nu). 
\end{equation}
Then there exists a constant $C=C(\a , \beta)$ such that for all $n=1,2,\dots $ we have
$$
a_n \le C A n^{-\a} .
$$
 \end{Lemma}
\begin{Remark}\label{R3.1} If conditions (\ref{3.20}) and (\ref{3.21}) are satisfied for $n\le N$ and $\nu\le N$ then the statement of Lemma \ref{L3.4} holds for $n\le N$. 
\end{Remark} 
 Suppose that
 $$
 a_\nu \ge A\nu^{1-q}.
 $$
 Then by (\ref{d3.3})
 $$
 a_{\nu+1}\le a_\nu(1-wa_\nu^{p-1})+c\nu^{-q} \le a_\nu(1-wa_\nu^{p-1}+(c/A)/\nu).
 $$
 Making $A$ large enough $A=C(a_0,c)w^{-\frac{1}{p-1}}$ we get $a_0<A$ and
 $$
 -wA^{p-1}+c/A \le -2.
 $$
 We now apply Lemma \ref{L3.4}.
 \end{proof}
This completes the proof of Lemma \ref{L3.2}
 \end{proof}
Applying Lemma \ref{L3.2} with $v=t$, $B=2^{1+q}\gamma$ we complete the proof of Theorem \ref{T3.1}.
\end{proof}
The following two theorems are the corresponding analogs of Theorems \ref{T3.0} and \ref{T3.1} for the REGA($\{\de_k\}$) instead of the WRGA($\{\de_k\}$). 
\begin{Theorem}\label{T3.0E} Let $E$ be a uniformly smooth on $A_1(\D)$ convex function. Suppose that a sequence  $\{\de_k\}$ is such that $\de_k\to0$ as $k\to\infty$.
Then for the REGA($\{\de_k\}$) we have 
$$
\lim_{m\to\infty}E(G_m) = \inf_{f\in A_1(\D)}E(x).    
$$
\end{Theorem}
\begin{Theorem}\label{T3.2} Let $E$ be a uniformly smooth on $A_1(\D)$ convex function with modulus of smoothness $\rho(E,u) \le \gamma u^q$, $1<q\le 2$. Then, for the REGA($\{\de_k\}$) with $\de_k\le c(k+1)^{-q}$ we have  
$$
E(G_m)-b \le  C(q,\gamma,E,c) m^{1-q},  
$$
where $b:=\inf_{f\in A_1(\D)}E(x)$.
\end{Theorem}
\begin{proof} By the definition of the REGA($\{\de_k\}$) we get
$$
E(G_m) \le \inf_{0\le \la\le 1;g\in\D}E((1-\la)G_{m-1} + \la g) +\de_{m-1}.
$$
Let $\ff_m^t$ be from the WRGA(co) with $\tau=\{t\}$. Then by Lemma \ref{L2.3} we obtain
$$
\inf_{0\le \la\le 1;g\in\D}E((1-\la)G_{m-1} + \la g) \le \inf_{0\le \la\le 1}E((1-\la)G_{m-1} + \la \ff_m^t)
$$
$$
\le E(G_{m-1} )+ \inf_{0\le\la\le 1}(-\la t (E(G_{m-1} )-b)+ 2\rho(E, 2\la)).
$$
This implies
\begin{equation}\label{3.11}
E(G_m) \le E(G_{m-1} )+ \inf_{0\le\la\le 1}(-\la  (E(G_{m-1} )-b)+ 2\rho(E, 2\la)) +\de_{m-1}.
\end{equation}
Inequality (\ref{3.11}) is a particular case of inequality (\ref{3.1}) from Lemma \ref{L3.1}. Thus, repeating the above proof of Theorems \ref{T3.0} and \ref{T3.1} we complete the proofs of Theorems \ref{T3.0E} and \ref{T3.2}.
\end{proof}

We now discuss an approximate version of the WGAFR(co), defined in the Introduction. 
 First, we prove a rate of convergence result. Denote
$$
D_1:=\{x:E(x)\le E(0) +1\}.
$$
Assume that $D_1$ is bounded.
\begin{Theorem}\label{T3.3}  Let $E$ be a uniformly smooth convex function with modulus of smoothness $\rho(E,D_{1},u)\le \gamma u^q$, $1<q\le 2$. Take a number $\e\ge 0$ and an element  $f^\e$ from $D$ such that
$$
E(f^\e) \le \inf_{x\in D}E(x)+ \e,\quad
f^\e/A(\e) \in A_1(\D),
$$
with some number $A(\e)\ge 1$.
Then we have for the WGAFR($\{\de_k\}$) with $t_k=t\in(0,1]$ and $\de_k\le c(k+1)^{-q}$, $k=0,1,\dots$,
$$
E(G_m)-\inf_{x\in D}E(x) \le  \e + C(E,q,\gamma, t,c)A(\e)^q m^{1-q} . 
$$
\end{Theorem}
\begin{proof} In the proof of Lemma 4.1 of \cite{T1} we established the inequality
$$
\inf_{ \la\ge 0,w}E((1-w)G_{m-1} + \la\varphi_m) \le E(G_{m-1})
$$
\begin{equation}\label{3.13}
   +\inf_{\la\ge 0}(-\la t_m A(\e)^{-1}(E(G_{m-1})-E(f^\e)) +2\rho(E,C_0\la)),\quad C_0=C(D),
\end{equation} 
under assumption that $\ff_m$ satisfies (\ref{3.12}) and $G_{m-1}\in D$. Clearly, (\ref{3.13}) holds if we replace $\inf_{\la\ge 0}$ by $\inf_{0\le\la\le 1}$ in the right hand side.

In the case of exact evaluations in the WGAFR(co) we had the monotonicity property $E(G_0)\ge E(G_1)\ge\cdots$ which implied that $G_n\in D$ for all $n$. In the case of the WGAFR($\{\de_k\}$) our assumption $\de_k\in[0,1]$ and choice of $w_m$ and $\la_m$ in (2) imply 
\begin{equation}\label{3.14}
E(G_m)\le E(0)+1,
\end{equation}
 which implies $G_n\in D_{1}$ for all $n$. 

Denote
$$
a_n:=\max(E(G_n)-E(f^\e),0).
$$
Note that we always have for $a_{m-1}\ge 0$
$$
a_{m-1} +\inf_{0\le\la\le 1}(-\la t A(\e)^{-1}a_{m-1} +2\gamma(C_0\la)^q) \ge 0.
$$
Therefore, inequality (\ref{3.13}) implies
\begin{equation}\label{3.17'}
a_m\le a_{m-1} +\inf_{0\le\la\le 1}(-\la t A(\e)^{-1}a_{m-1} +2\gamma(C_0\la)^q) +\de_{m-1}.
\end{equation}
It is similar to (\ref{3.4}).   We apply Lemma 
\ref{L3.2} with $v=t A(\e)^{-1}$, $B=2\gamma C_0^q$ and complete the proof.
\end{proof}

\begin{Corollary}\label{C3.1} Under conditions of Theorem \ref{T3.3}, specifying 
$$
A(\e):=\inf\{M: \exists f : f/M \in A_1(\D),\quad  E(f)\le \inf_{x\in D}E(x) +\e\},
$$
and denoting
$$
\e_m := \inf\{\e: A(\e)^qm^{1-q}\le\e\},
$$
we obtain
$$
E(G_m) - \inf_{x\in D}E(x) \le C(E,q,\gamma,t)\e_m.
$$
\end{Corollary}

Second, we prove a convergence result.
\begin{Theorem}\label{T3.0FR} Let $E$ be a uniformly smooth on $D_1$ convex function. Suppose that a sequence  $\{\de_k\}$ is such that $\de_k\to0$ as $k\to\infty$.
Then for the WGAFR($\{\de_k\}$) with $t_k=t$, $t\in (0,1]$, $k=1,2,\dots$, we have 
$$
\lim_{m\to\infty}E(G_m) = \inf_{x\in  D}E(x).    
$$
\end{Theorem}
\begin{proof} The proof is similar to the proof of Theorem \ref{T3.0} with modifications as in the above proof of Theorem \ref{T3.3}. Take an arbitrary $\e>0$ and let $f^\e$ to be as above in Theorem \ref{T3.3}. In the same way as we obtained (\ref{3.17'}) we get
\begin{equation}\label{3.18'}
a_m\le a_{m-1} +\inf_{0\le\la\le 1}(-\la t A(\e)^{-1}a_{m-1} +2\rho(E,C_0\la)) +\de_{m-1}.
\end{equation}
We now apply Lemma \ref{L3.0} with $v:=tA(\e)^{-1}$, $B:=2$, $\rho(u):=\rho(E,C_0\la)$ and obtain that 
\begin{equation}\label{3.19'}
\lim_{m\to\infty}a_m=0.
\end{equation}
Relation (\ref{3.19'}) implies that
\begin{equation}\label{3.20'}
\limsup_{m\to\infty}(E(G_m)-\inf_{x\in D}E(x) )\le \e.
\end{equation}
This implies in turn that
\begin{equation}\label{3.21'}
\lim_{m\to\infty}(E(G_m)-\inf_{x\in D}E(x) )=0.
\end{equation}
\end{proof}

The algorithm REGA($\{\de_k\}$) is the function evaluation companion of the WRGA($\{\de_k\}$). The following function evaluation companion of the \newline WGAFR($\{\de_k\}$) was introduced and studied in \cite{DT} in the case $\de_k=\de >0$, $k=0,1,\dots$.

{\bf $E$-Greedy Algorithm with Free Relaxation and errors $\{\de_k\}$ \newline (EGAFR($\{\de_k\}$)).} 
Let   $\de_k\in[0,1]$, $k=0,1,2,\dots$. We define   $G_0  := 0$. Then for each $m\ge 1$ we have the following inductive definition.

  Find $\varphi_m\in\D$, $w_m$, and $ \lambda_m$ such that
$$
E((1-w_m)G_{m-1} + \la_m\varphi_m) \le \inf_{g\in\D;\la,w}E((1-w)G_{m-1} + \la g) +\de_{m-1}
$$
and define
$$
G_m:=   (1-w_m)G_{m-1} + \la_m\varphi_m.
$$
 
 In the same way as Theorems \ref{T3.0E} and \ref{T3.2} for the REGA($\{\de_k\}$) were derived from the proofs of Theorems \ref{T3.0} and \ref{T3.1} one can derive analogs of Theorems \ref{T3.0FR} and \ref{T3.3} for the EGAFR($\{\de_k\}$). 
 
 \begin{Proposition}\label{P3.1}  Theorems \ref{T3.0FR} and \ref{T3.3} hold for the EGAFR($\{\de_k\}$). 
 \end{Proposition}


\begin{thebibliography}{9999}

\bibitem{B} A.R. Barron, Universal approximation bounds for superposition of $n$ sigmoidal functions, IEEE Trans. Inf. Theory, {\bf 39} (1993), 930--945.





\bibitem{CRPW} V. Chandrasekaran, B. Recht, P.A. Parrilo, and A.S. Willsky, The convex geometry of linear inverse problems, Proceedings of the 48th Annual Allerton Conference on Communication, Control and Computing, 2010, 699--703.

\bibitem{Cl} K.L. Clarkson, Coresets, Sparse Greedy Approximation, and the Frank-Wolfe Algorithm, ACM Transactions on Algorithms, {\bf 6} (2010), Article No. 63.

\bibitem{DR} V.F. Demyanov and A.M. Rubinov, Approximate methods in optimization problems, Elsevier, 1970.

   
\bibitem{DT} R.A. DeVore and V.N. Temlyakov, Some remarks on greedy algorithms in optimization, IMI Preprint, 2014:01, 1--20; arXiv:1401.0334v1, 1 Jan 2014.   
   
\bibitem{DHM} M. Dudik, Z. Harchaoui, and J. Malick, Lifted coordinate descent for learning with trace-norm regularization, In AISTATS, 2012. 

\bibitem{DH} J.C. Dunn and S. Harshbarger,  Conditional gradient algo-
rithms with open loop step size rules. Journal of Mathe-
matical Analysis and Applications, {\bf 62} (1978), 432Ð444.
   

\bibitem{FW} M. Frank and P. Wolfe, An algorithm for quadratic programming, 
Naval Research Logistics Quarterly, {\bf 3} (1956), 95--110.

\bibitem{Ja1} M. Jaggi, Sparse Convex Optimization Methods for Ma-
chine Learning, PhD thesis, ETH Z{\"u}rich, 2011.

\bibitem{Ja2} M. Jaggi, Revisiting Frank-Wolfe: Projection-Free Sparse Convex Optimization, Proceedings of the 30$^{th}$ International Conference on Machine Learning, Atlanta, Georgia, USA, 2013.

\bibitem{JS} M. Jaggi and M. Sulovsk{\'y},  A Simple Algorithm for Nuclear
Norm Regularized Problems. ICML, 2010.

\bibitem{Jo} L.K Jones,  A Simple Lemma on Greedy Approximation in Hilbert Space and Convergence Rates for Projection Pursuit Regression and Neural Network Training. The Annals of Statistics, {\bf 20} (1992), 608Ð613.

   

   
\bibitem{SSZ} S. Shalev-Shwartz, N. Srebro, and T. Zhang, Trading accuracy for sparsity in optimization problems with sparsity constrains, SIAM Journal on Optimization, {\bf 20(6)} (2010), 2807--2832.

\bibitem{T1a} V.N. Temlyakov, Greedy Algorithms and $m$-term Approximation With Regard to Redundant Dictionaries, J. Approx. Theory {\bf 98} (1999),  117--145.


   
  
  

\bibitem{T7} V.N. Temlyakov, Greedy-Type Approximation in Banach Spaces and Applications,  Constr. Approx.,   {\bf 21} (2005),    257--292.


\bibitem{Tbook} V.N. Temlyakov, Greedy approximation, Cambridge University Press, 2011.

\bibitem{T1} V.N. Temlyakov, Greedy approximation in convex optimization, 
IMI Preprint, 2012:03, 1--25;  arXiv:1206.0392v1, 2 Jun 2012.

\bibitem{T2} V.N. Temlyakov, Greedy expansions in convex optimization, 
Proceedings of the Steklov Institute of Mathematics, {\bf 284} (2014), 244--262
  (arXiv:1206.0393v1, 2 Jun 2012).

\bibitem{TRD} A. Tewari, P. Ravikumar, and I.S. Dhillon, Greedy Algorithms for Structurally Constrained High Dimensional Problems, prerint, (2012), 1--10.

\bibitem{Z} T. Zhang, Sequential greedy approximation for certain convex optimization problems, IEEE Transactions on Information Theory, {\bf 49(3)} (2003), 682--691.

\end{thebibliography}
\end{document}